\documentclass[letterpaper]{article}
\usepackage{uai2019}
\usepackage[margin=1in]{geometry}
\usepackage{hyperref}
\usepackage{changes}
\usepackage{url}
\usepackage{booktabs}
\usepackage{amssymb}
\usepackage{enumitem}
\usepackage{graphicx}
\usepackage{color}
\usepackage{subcaption}
\usepackage{hyperref}
\usepackage{amsmath}
\usepackage{amssymb}
\usepackage{amsthm}
\usepackage{amsfonts}       % blackboard math symbols
\usepackage{nicefrac}       % compact symbols for 1/2, etc.
\usepackage{microtype}      % microtypography
\usepackage{graphicx}
\usepackage{tabularx}
\usepackage{fancyvrb}
\usepackage{comment}
\usepackage{algorithm}
\usepackage{algorithmic}
\usepackage{float}
\usepackage{wrapfig}
\usepackage{listings}
\usepackage{times}
\usepackage{thmtools}
\usepackage{thm-restate}
\usepackage[round,sort,comma,super,authoryear]{natbib}
\usepackage{multirow} % For table rows across > 1 line.

\theoremstyle{definition}
\newtheorem{definition}{Definition}[section]

\newtheorem{theorem}{Theorem}[section]
\newtheorem{lemma}{Lemma}[section]

\newtheorem{fact}{Fact}[section]

\newcommand{\vect}[1]{\text{vec}(#1)}
\newcommand{\rank}[1]{\text{rank}(#1)}
\newcommand{\RR}{\mathbb{R}}

\newcommand{\SPD}{\mathbb{S}_+}
\newcommand{\PD}{\mathbb{S}_{++}}

\DeclareMathOperator*{\argmin}{arg\,min}

\newcommand{\diag}{\text{diag}}
\newcommand{\tr}{\text{tr}}

\definecolor{dkgreen}{rgb}{0,0.6,0}
\definecolor{gray}{rgb}{0.5,0.5,0.5}
\definecolor{mauve}{rgb}{0.58,0,0.82}

\title{Efficient Multitask Feature and Relationship Learning}

\author{ \textbf{Han Zhao} \and \text{Otilia Stretcu} \\
Machine Learning Department \\
Carnegie Mellon University\\
\texttt{\{han.zhao, ostretcu\}@cs.cmu.edu}
\And
\textbf{Alexander J. Smola}   \\
Amazon Web Services \\
\texttt{alex@smola.org} \\
\And
\textbf{Geoffrey J. Gordon} \\
Microsoft Research Montreal \\
Carnegie Mellon University \\
\texttt{geoff.gordon@microsoft.com}
}

\begin{document}

\maketitle

\begin{abstract}
We consider a multitask learning problem, in which several predictors are learned jointly. Prior research has shown that learning the relations between tasks, and between the input features, together with the predictor, can lead to better generalization and interpretability, which proved to be useful for applications in many domains. In this paper, we consider a formulation of multitask learning that learns the relationships both between tasks and between features, represented through a task covariance and a feature covariance matrix, respectively. First, we demonstrate that existing methods proposed for this problem present an issue that may lead to ill-posed optimization. We then propose an alternative formulation, as well as an efficient algorithm to optimize it. Using ideas from optimization and graph theory, we propose an efficient coordinate-wise minimization algorithm that has a closed form solution for each block subproblem. Our experiments show that the proposed optimization method is orders of magnitude faster than its competitors. We also provide a nonlinear extension that is able to achieve better generalization than existing methods.
\end{abstract}

\section{INTRODUCTION}
In machine learning the goal is often to train predictive models for one or more tasks of interest. Making accurate predictions relies heavily on the existence of labeled data for the desired tasks. However, in real-world problems data is often hard to acquire (e.g., medical domains) or expensive to label (e.g., image segmentation). For many tasks, this makes it impractical or impossible to collect large volumes of labeled data. Multitask learning is a sub-problem of the general transfer learning paradigm that aims to improve generalization performance in a learning task, by learning models for multiple related tasks simultaneously. It has received considerable interest in the past decades~\citep{caruana1997multitask,evgeniou2004regularized,argyriou2007spectral,
argyriou2008convex,kato2008multi,liu2009multi,jacob2009clustered,zhang2010convex,zhang2010learning,
chen2011integrating,chen2012learning,li2015multi,jawanpuria2015efficient,adel2017unsupervised,zhao2017multiple,zhao2018adversarial,zhao2019learning,zhao2019deep}. One of the underlying assumptions behind many multitask learning algorithms is that the tasks are related to each other. Hence, a key question is how to define the notion of task relatedness, and how to capture it in the learning formulation. A common assumption is that tasks can be described by weight vectors, and that they are sampled from a shared prior distribution over their space~\citep{liu2009multi,zhang2010convex,zhang2010multi}. Another strand of work assumes common feature representations to be shared among multiple tasks, and the goal is to learn the shared representation as well as task-specific parameters simultaneously~\citep{thrun1996explanation,caruana1997multitask,evgeniou2007multi,argyriou2008convex}. Moreover, when structure about multiple tasks is available, e.g., task-specific descriptors~\citep{bonilla2007kernel} or a task similarity graph~\citep{evgeniou2004regularized}, regularizers can often be incorporated into the learning formulation to penalize hypotheses that are not consistent with the given structure. Very recently, \citet{sener2018multi} tackle the problem of multitask learning from the perspective of multi-objective optimization. Specifically, this work aims to find a Pareto-optimal solution for the multi-objective function defined by multiple tasks, and proposes to use Frank-Wolfe algorithm to find gradient update for shared parameters. 

There have been several attempts to improve predictions by either learning the relationships between different tasks~\citep{zhang2010convex}, or by exploiting the relationships between different features~\citep{argyriou2008convex}. In this paper we consider a multiconvex framework for multitask learning that improves predictions over tabula rasa learning by assuming that all the task vectors are sampled from a common matrix-variate normal prior. The framework, known as MTFRL~\citep{zhang2010learning}, learns the relationships both between tasks and between features simultaneously via two covariance matrices, i.e., the feature covariance matrix and the task covariance matrix. In this context, learning multiple tasks corresponds to estimating a matrix of model parameters, and learning feature/task relationships corresponds to estimating the row/column covariance matrices of model parameters, respectively. This property is favorable for applications where we not only aim for better generalization, but also seek to have a clear understanding about the relationships among different tasks. 

The goal of MTFRL is to optimize over both the task vectors, as well as the two covariance matrices in the prior. When the loss function is convex, the regularized problem of MTFRL is multiconvex. Previous approaches~\citep{zhang2010learning, zhang2011supervision, long2017learning} for solving this problem hinge on the classic flip-flop algorithm~\citep{dutilleul1999mle} to estimate the two covariance matrices. However, as we point out in Section~\ref{sec:problem}, the flip-flop algorithm cannot be directly applied as the maximum likelihood estimation (MLE) formulation of the multitask learning problem under this setting is ill-posed. As a result, in practice, heuristics have to be invented and applied in the algorithm to ensure the positive-definiteness of both covariance matrices. However, it is not clear whether such a fixed algorithm still converges or not. 

In this paper we propose a well-defined variant of the MTFRL framework, and design a block coordinate-wise minimization algorithm to solve this problem. We term our new formulation FEaTure and Relation learning (FETR). By design, FETR is free of the nonpositive-definite problem in MTFRL. To solve FETR, we propose efficient and analytic solutions for each of the subproblems, which allows us to get rid of the expensive iterative procedure to optimize the covariance matrices. Specifically, we achieve this by reducing an underlying matrix optimization problem with positive definite constraints into a minimum weight perfect matching problem on a complete bipartite graph, where we are able to solve analytically using combinatorial techniques. To solve the weight learning subproblem, we propose three different strategies, including a closed form solution, a gradient descent method with linear convergence guarantees when the instances are not shared by multiple tasks, and a numerical solution based on Sylvester equation when instances are shared. 

We demonstrate the efficiency of the proposed optimization algorithm by comparing it with an off-the-shelf projected gradient descent algorithm and the classic flip-flop algorithm, on both synthetic and real-world data. Experiments show that the proposed optimization method is orders of magnitude faster than its competitors, and it often converges to better solutions. Lastly, we extend FETR to nonlinear setting by combining its regularization scheme with rich nonlinear transformations using neural networks. This combined approach is able to achieve significantly better generalizations than existing methods on real-world datasets.

To summarize, our contributions are three-fold:
\begin{itemize}[topsep=0pt,itemsep=0pt]
\itemsep0em 
	\item 	We point out an ill-posed MLE problem of the existing multitask learning formulations and propose a well-defined variant, termed as FETR.
	\item 	To optimize FETR, we design an efficient block coordinate-wise minimization algorithm and derive analytic solutions for each of the subproblems.
	\item 	We extend our FETR formulation to nonlinear settings and empirically demonstrate its better generalizations on real-world datasets.
\end{itemize}

\section{PRELIMINARY}
\label{sec:prelimiary}
We start by introducing notations used throughout the paper and briefly discussing the MTFRL framework~\citep{zhang2010learning}. %Readers are referred to~\citet{zhang2017survey} for a recent survey on various multitask learning algorithms. 

\subsection{NOTATION AND SETUP}
We use lowercase letters, such as $y$, to represent scalars, and lowercase bold letters, such as $\mathbf{x}$, to denote vectors. Capital letters are reserved for matrices. We use $\SPD^m$ and $\PD^m$ to denote the $m$-dimensional symmetric positive semidefinite cone and the $m$-dimensional symmetric positive definite cone, respectively. We write $\tr(A)$ for the trace of a matrix $A$, and $\mathcal{N}(\mathbf{m}, \Sigma)$ for the multivariate normal distribution with mean $\mathbf{m}$ and covariance matrix $\Sigma$. Finally, $G = (A, B, E; w)$ is a weighted bipartite graph with vertex sets $A$, $B$, edge set $E$ and weight function $w: E\to \RR_+$. For a matrix $W\in\RR^{d\times m}$, we use $\vect{W}\in\RR^{dm}$ to denote its vectorization. We consider the following setup. Suppose we are given $m$ learning tasks $\{T_i\}_{i=1}^m$, where for each learning task $T_i$ we have access to a training set $\mathcal{D}_i$ with $n_i$ data instances $(\mathbf{x}_i^j, y_i^j), j\in[n_i]$. For the simplicity of discussion, here we focus on the regression setting where $\mathbf{x}_i^j\in\mathcal{X}_i\subseteq\RR^d$ and $y_i^j\in\RR$. Extension to classification setting is straightforward. Let $f_i(\mathbf{w}_i, \cdot):\mathcal{X}_i\to\RR$ be our model with parameter $\mathbf{w}_i$. In what follows, we will assume our model for each task $T_i$ to be a linear regression, i.e., $f_i(\mathbf{w}_i, \mathbf{x}) = \mathbf{w}_i^T\mathbf{x}$. 

\subsection{MATRIX-VARIATE NORMAL DISTRIBUTION}
A matrix-variate normal distribution~\citep{gupta1999matrix} $W\sim\mathcal{MN}_{d\times m}(M, A, B)$ with mean $M\in\RR^{d\times m}$, row covariance matrix $A\in\PD^{d}$ and column covariance matrix $B\in\PD^m$ can be understood as a multivariate normal distribution with $\vect{W}\sim\mathcal{N}(\vect{M}, A\otimes B)$.\footnote{Probability density:
$p(X) = \exp(-\frac{1}{2}\tr(A^{-1}(X-M)B^{-1}(X-M)^T))/(2\pi)^{md/2}|A|^{m/2}|B|^{d/2}$.}
One advantage of the matrix-variate normal distribution over its equivalent multivariate counterpart is that by imposing structure on the row and column covariance matrices, the former admits a much more compact representation than the latter ($O(m^2 + d^2)$ versus $O(m^2d^2)$). The MLE of the matrix-variate normal distribution has been well studied in the literature~\citep{dutilleul1999mle}. Specifically, given an i.i.d. sample $\{W_i\}_{i=1}^n$ from $\mathcal{MN}_{d\times m}(M, A, B)$, the MLE of $M$ is $\overline{W} = \sum_{i=1}^n W_i/n$. The MLE of $A$ and $B$ are solutions to the following system:
\begin{equation}
\begin{cases}
A = \frac{1}{nm}\sum_{i=1}^n(W_i - \overline{W}) B^{-1} (W_i - \overline{W})^T \\
B = \frac{1}{nd}\sum_{i=1}^n (W_i - \overline{W})^T A^{-1} (W_i - \overline{W})
\end{cases}
\label{equ:ff}
\end{equation}
The above system of equations does not have a closed form solution as the two covariance estimates depend on each other. Hence, their estimates must be computed in an iterative fashion until convergence, which is known as the ``flip-flop'' algorithm~\citep{dutilleul1999mle,glanz2013expectation}. Furthermore, \citet{dutilleul1999mle} showed that the flip-flop algorithm is guaranteed to converge to positive definite covariance matrices iff $n\geq \max(d/m, m/d) + 1$. More properties of the MLE of the matrix-variate normal distribution can be found in~\citep{ros2016existence}.

\subsection{MULTITASK FEATURE AND RELATIONSHIP LEARNING}
In linear regression, the likelihood function for task $i$ is given by: $y_i^j \mid \mathbf{x}_i^j, \mathbf{w}_i, \epsilon_i \sim \mathcal{N}(\mathbf{w}_i^T\mathbf{x}, \epsilon_i^2)
$. Let $W = (\mathbf{w}_1, \ldots, \mathbf{w}_m)\in\RR^{d\times m}$ be the model parameter for $m$ different tasks drawn from the matrix-variate normal distribution $\mathcal{MN}_{d\times m}(W~|~\mathbf{0}_{d\times m}, \Sigma_1^{-1}, \Sigma_2^{-1})$. By maximizing the joint distribution and optimize over both the model parameters, as well as the two covariance matrices in the prior, we reach the following optimization problem: 
\begin{align}
& \underset{{W, \Sigma_1, \Sigma_2}}{\text{minimize}} && \sum_{i=1}^m\sum_{j=1}^{n_i}(y_i^j - \mathbf{w}_i^T\mathbf{x}_i^j)^2  + \eta~\tr(\Sigma_1 W\Sigma_2 W^T) \nonumber\\
& && - \eta~(m\log|\Sigma_1| + d\log|\Sigma_2|) \nonumber\\
& \text{subject to} && \Sigma_1\succ 0, \Sigma_2 \succ 0
\label{equ:popt}
\end{align}
where $\Sigma_1\in \PD^d, \Sigma_2\in\PD^m$ are the row and column precision matrices of the matrix normal prior distribution, respectively, and $\eta$ is a constant that does not depend on the optimization variables. It is not hard to see that the optimization problem in (\ref{equ:popt}) is not convex due to the coupling between $W, \Sigma_1$ and $\Sigma_2$ in the trace term. On the other hand, since the $\log|\cdot|$ function is concave in the positive definite cone~\citep{boyd2004convex}, and the trace is linear in terms of its components, it follows that (\ref{equ:popt}) is multiconvex. \citet{zhang2010learning} propose to use the flip-flop algorithm to solve the matrix subproblem in \eqref{equ:popt}, and this approach has also been widely applied in following publications on multitask learning~\citep{zhang2011supervision,li2014bayesian,long2017learning}. 

\section{ILL-POSED OPTIMIZATION}
\label{sec:problem}
In this section we first point out an important issue in the literature on the application of the flip-flop algorithm to solve the matrix subproblem in \eqref{equ:popt}. We then proceed to propose a well-defined variant of \eqref{equ:popt} to fix the problem. Interestingly, the variant we propose admits a closed form solution for each block variable that can be computed efficiently without any iterative procedure, which we will describe and derive in more detail in Section~\ref{sec:opt}. 

As proved by \citet{dutilleul1999mle}, one sufficient and necessary condition for the flip-flop algorithm to converge to positive definite matrices is that the number of samples from the matrix-variate normal distribution should satisfy $n > \max(d/m, m/d)$. However, in the context of multitask learning, we are essentially dealing with an inference problem, where the goal is to estimate the value of $W$, which is assumed to be an unknown but unique model parameter from the prior. This means that in this case we have $n = 1$, hence the condition for the convergence of the algorithm is violated. Technically, for any $W\in\RR^{d\times m}$ where $d\neq m$, following the iterative update formula of the flip-flop algorithm in \eqref{equ:ff}, for any feasible initialization of $\Sigma_1^{(0)}\in\PD^d$ and $\Sigma_2^{(0)}\in\PD^m$, we will have $\Sigma_1^{(1)} = W (\Sigma_2^{(0)})^{-1} W^T / m, 
\Sigma_2^{(1)} = W^T (\Sigma_1^{(0)})^{-1} W / d$. Now since $d\neq m$, we know that 
$\rank{W}\leq \min\{d, m\} < \max\{d, m\}$. As a result, after one iteration, we will have
\begin{align*}
\rank{\Sigma_1^{(1)}} \leq \rank{W}  < \max\{d, m\}\\
\rank{\Sigma_2^{(1)}} \leq \rank{W^T} < \max\{d, m\}
\end{align*}
i.e., at least one of $\Sigma_1^{(1)}$ and $\Sigma_2^{(1)}$ is going to be rank deficient, and in the next iteration the inverse operation is not well-defined on at least one of them. As a fix, \citet{zhang2010learning} proposed to use an artificial fudge factor to ensure that both covariance matrices stay positive definite after each update:
\begin{align*}
\Sigma_1^{(t+1)} = W (\Sigma_2^{(t)})^{-1} W^T / m + \epsilon I_d\\
\Sigma_2^{(t+1)} = W^T (\Sigma_1^{(t)})^{-1} W / d  + \epsilon I_m
\end{align*}
where $\epsilon > 0$ is a fixed, small constant. However, since the fudge factor $\epsilon$ is a fixed constant which does not decrease to 0 in the limit, it introduces extra biases into the estimation, and thus it is not clear whether or not the fixed algorithm converges.

Perhaps what is more surprising is that \eqref{equ:popt} is not even well-defined as an optimization problem. As a counterexample, we can fix $W = \mathbf{0}_{d\times m}$ and let $\Sigma_{1, \sigma} = \sigma I_d$, $\Sigma_{2, \sigma} = \sigma I_m$ with $\sigma > 0$ so that both $\Sigma_{1, \sigma}$ and $\Sigma_{2, \sigma}$ are feasible. Now let $\sigma \to \infty$, and it is easy to verify that in this case the objective function goes to $-\infty$. Although we only provide one counterexample, there is no reason to believe that the one we find is the only case where \eqref{equ:popt} fails. In fact, \citet{ros2016existence} have recently shown that the MLE of $\Sigma_1\otimes \Sigma_2$ does not exist if $n \leq \max\{d/m, m/d\}$, and the only nontrivial sufficient condition known so far to guarantee the existence of the MLE is $n > md$. However, in the context of multitask learning, the unknown model parameter $W$ is unique and hence we have $n = 1\ll md$. 

Given the wide applications of the above multitask learning framework in the literature, as well as the flip-flop algorithm in this setting, we feel it important and urgent to solve the above ill-posed and nonpositive definite problem. To this end, for some positive constants $0 < l < u$, we propose a variant of \eqref{equ:popt} as follows:
\begin{align}
& \underset{\Sigma_1, \Sigma_2, W}{\text{minimize}} && \sum_{i=1}^m\sum_{j=1}^{n_i}(y_i^j - \mathbf{w}_i^T\mathbf{x}_i^j )^2 +  \eta~\tr(\Sigma_1W\Sigma_2 W^T) \nonumber \\
& && - \eta~(m\log|\Sigma_1| + d\log|\Sigma_2|) \nonumber\\
& \text{subject to} && l I_d \preceq\Sigma_1\preceq u I_d, l I_m \preceq \Sigma_2\preceq u I_m
\label{equ:nopt}
\end{align}
The bounded constraints make the feasible set compact. Since the objective function is continuous, by the extreme value theorem, we know that the matrix subproblem of \eqref{equ:nopt} becomes well-defined and can achieve finite lower and upper bounds within the feasible set. Alternatively, one can also understand this constraint as specifying a truncated matrix normal prior over the compact set. As we will see shortly, technically the bounded constraint also allows us to develop an optimization procedure for $W$ with linear convergence rate, which is an exponential acceleration over the unbounded case. 

\section{MULTICONVEX OPTIMIZATION}
\label{sec:opt}
In this section we propose a block coordinate-wise minimization algorithm to optimize the objective given in (\ref{equ:nopt}). In each iteration, we alternatively minimize over $W$ with $\Sigma_1$ and $\Sigma_2$ fixed, then minimize over $\Sigma_1$ with $W$ and $\Sigma_2$ fixed, and lastly minimize $\Sigma_2$ with $W$ and $\Sigma_1$ fixed. The whole procedure is repeated until a stationary point is found. Due to space limit, we defer all the proofs and derivations to appendix. To simplify the notation, we assume $n = n_i, \forall i\in [m]$. Let $Y = (\mathbf{y}_1, \ldots, \mathbf{y}_m)\in\RR^{n\times m}$ be the target matrix and $X \in\RR^{n\times d}$ be the feature matrix shared by all the tasks. Using this notation, the objective can be equivalently expressed in matrix form as:
\begin{align}
& \underset{\Sigma_1, \Sigma_2, W}{\text{minimize}} && ||Y - XW||_F^2 + \eta~||\Sigma_1^{1/2}W\Sigma_2^{1/2}||_F^2 \nonumber\\
& && - \eta~(m\log|\Sigma_1| + d\log|\Sigma_2|)\nonumber \\
& \text{subject to} && l I_d \preceq\Sigma_1\preceq u I_d, l I_m \preceq \Sigma_2\preceq u I_m
\label{equ:matrixopt}
\end{align}

\subsection{OPTIMIZATION OF $W$}
In order to minimize over $W$ when both $\Sigma_1$ and $\Sigma_2$ are fixed, we solve the following subproblem:
\begin{align}
\underset{W}{\text{minimize}}\quad  h(W) \triangleq  ||Y - XW||_F^2 + \eta~||\Sigma_1^{1/2}W\Sigma_2^{1/2}||_F^2
\label{equ:optw}
\end{align}
As shown in the last section, this is an unconstrained convex optimization problem. We present three different algorithms to find the optimal solution of this subproblem. The first one guarantees to find an exact solution in closed form in $O(m^3d^3)$ time. The second one does gradient descent with fixed step size to iteratively refine the solution, and we show that in our case a linear convergence rate can be guaranteed. The third one finds the optimal solution by solving the Sylvester equation~\citep{bartels1972solution} characterized by the first-order optimality condition, after a proper transformation.

\textbf{A closed form solution}. It is worth noting that it is not obvious how to obtain a closed form solution directly from the formulation in (\ref{equ:optw}). An application of the first order optimality condition to (\ref{equ:optw}) will lead to: $X^TXW + \eta~\Sigma_1 W\Sigma_2 = X^TY$. Hence except for the special case where $\Sigma_2 = cI_m$ with $c > 0$ a constant, the above equation does not admit an easy closed form solution in its matrix representation. The workaround is based on the fact that the $d\times m$ dimensional matrix space is isomorphic to the $dm$ dimensional vector space, with the $\text{vec}(\cdot)$ operator implementing the isomorphism from $\RR^{d\times m}$ to $\RR^{dm}$. Using this property, we have:
\begin{restatable}{proposition}{closed}
\label{claim:closed}
(\ref{equ:optw}) can be solved in closed form in $O(m^3d^3 + mnd^2)$ time; the optimal solution $W^*$ is: $\vect{W^*} = \left(I_m\otimes (X^TX) + \eta~\Sigma_2\otimes \Sigma_1\right)^{-1}\text{vec}(X^TY)$.
\end{restatable}
The computational bottleneck in the above procedure is in solving an $md\times md$ system of equations, which scales as $O(m^3d^3)$ if no further sparsity structure is available. 

\textbf{Gradient descent}. The closed form solution shown above scales cubically in both $m$ and $d$, and requires us to explicitly form a matrix of size $md\times md$. This can be intractable even for moderate $m$ and $d$. In such cases, instead of computing an exact solution to (\ref{equ:optw}), we can use gradient descent with fixed step size to obtain an approximate solution. The objective function $h(W)$ in (\ref{equ:optw}) is differentiable and its gradient can be obtained in $O(m^2d + md^2)$ time as $\nabla_W h(W) = X^T(Y - XW) + \eta~\Sigma_1W\Sigma_2$. Note that we can compute in advance both $X^TY$ and $X^TX$ in $O(nd^2)$ time, and cache them so that we do not need to recompute them in each gradient update step. Let $\lambda_i(A)$ be the $i$th largest eigenvalue of a real symmetric matrix $A$. Adapted from~\citet{nesterov2013introductory}, we provide a linear convergence guarantee for the gradient method in the following proposition:
\begin{restatable}{proposition}{gdconverge}
\label{thm:wconverge}
Let $\lambda_{l} = \lambda_d(X^TX) + \eta l^2$ and $\lambda_u = \lambda_1(X^TX)  + \eta u^2$. Choose $0 < t \leq \frac{2}{\lambda_u + \lambda_l}$. For all $\varepsilon > 0$,  gradient descent with step size $t$ converges to the optima within $O(\log(1/\varepsilon))$ steps.
\end{restatable}
The computational complexity to achieve an $\varepsilon$ approximate solution using gradient descent is $O(nd^2 + \log(1/\varepsilon)(m^2d + md^2))$. Compared with the $O(m^3d^3 + mnd^2)$ complexity for the exact solution, the gradient descent algorithm scales much better provided the condition number $\kappa \triangleq \lambda_u / \lambda_l$ is not too large. As a side note, when the condition number is large, we can effectively reduce it to $\sqrt{\kappa}$ by using conjugate gradient method~\citep{shewchuk1994introduction}. 

\textbf{Sylvester equation}. In the field of control theory, a Sylvester equation~\citep{bhatia1997and} is a matrix equation of the form $AX + XB = C$, where the goal is to find a solution matrix $X$ given $A, B$ and $C$. For this problem, there are efficient numerical algorithms with highly optimized implementations that can obtain a solution within cubic time. For example, the Bartels-Stewart algorithm~\citep{bartels1972solution} solves the Sylvester equation by first transforming $A$ and $B$ into Schur forms by QR factorization, and then solves the resulting triangular system via back-substitution. Our third approach is based on the observation that we can equivalently transform the first-order optimality equation into a Sylvester equation by multiplying both sides of the equation by $\Sigma_1^{-1}$: $\Sigma_1^{-1}X^TXW + \eta~W\Sigma_2 = \Sigma_1^{-1}X^TY$. As a result, finding the optimal solution of the subproblem amounts to solving the above Sylvester equation. Specifically, the solution to the above equation can be obtained using the Bartels-Stewart algorithm in $O(m^3 + d^3 + nd^2)$. 

\textbf{Remark}. Both the gradient descent and the Bartels-Stewart algorithm find the optimal solution in cubic time. However, gradient descent is more widely applicable than the Bartels-Stewart algorithm: the Bartels-Stewart algorithm only applies to the case where all the tasks share the same instances, so that we can write down the matrix equation explicitly, while gradient descent can be applied in the case where each task has different number of inputs and those inputs are not shared among tasks. On the other hand, as we will see in the experiments, in practice the Bartels-Stewart algorithm is faster than gradient descent, and provides a more numerically stable solution.

\subsection{OPTIMIZATION OF $\Sigma_1$ AND $\Sigma_2$}
\begin{algorithm}[htb]
\caption{Minimize $\Sigma_1$}
\label{alg:minsigma1}
\centering
\begin{algorithmic}[1]
\REQUIRE   $W$, $\Sigma_2$ and $l, u$.
\STATE  $[V, \nu]\leftarrow$ SVD$(W\Sigma_2W^T)$.
\STATE  $\lambda\leftarrow \mathbb{T}_{[l, u]}(m/\nu)$.
\STATE  $\Sigma_1\leftarrow V\diag(\lambda)V^T$.
\end{algorithmic}
\end{algorithm}
\begin{algorithm}[htb]
\caption{Minimize $\Sigma_2$}
\label{alg:minsigma2}
\centering
\begin{algorithmic}[1]
\REQUIRE   $W$, $\Sigma_1$ and $l, u$.
\STATE  $[V, \nu]\leftarrow$ SVD$(W^T\Sigma_1W)$.
\STATE  $\lambda\leftarrow \mathbb{T}_{[l, u]}(d/\nu)$.
\STATE  $\Sigma_2\leftarrow V\diag(\lambda)V^T$.
\end{algorithmic}
\end{algorithm}

Before we delve into the detailed analysis below, we first list the final algorithms used to optimize $\Sigma_1$ and $\Sigma_2$ in Algorithm~\ref{alg:minsigma1} and Algorithm~\ref{alg:minsigma2}, respectively. The hard-thresholding function used in Line 2 of Algorithm~\ref{alg:minsigma1} and Algorithm~\ref{alg:minsigma2} is defined as follows:
\begin{equation}
\mathbb{T}_{[l, u]}(x) = \max\{l, \min\{u, x\}\}
\label{equ:hardthresholding}
\end{equation}
The hard-thresholding function essentially keeps the value of its argument $x$ if $l\leq x\leq u$, otherwise it truncates the value of $x$ to $l (u)$ if $x < l (x > u)$ respectively. Both algorithms are remarkably simple: each algorithm only involves one SVD, one truncation and two matrix multiplications. The computational complexities of Algorithm~\ref{alg:minsigma1} and Algorithm~\ref{alg:minsigma2} are bounded by $O(m^2d + md^2 + d^3)$ and $O(m^2d + md^2 + m^3)$, respectively. 

In what follows we focus on analyzing the optimization w.r.t.\ $\Sigma_1$. A symmetric analysis can be applied to solve $\Sigma_2$ as well. In order to minimize over $\Sigma_1$ when $W$ and $\Sigma_2$ are fixed, we solve the following subproblem:
\begin{equation}
\underset{l I_d \preceq \Sigma_1 \preceq u I_d}{\text{minimize}} \quad \tr(\Sigma_1 W\Sigma_2 W^T) - m\log|\Sigma_1| 
\label{equ:optsigma}
\end{equation}
Although (\ref{equ:optsigma}) is a convex optimization problem, it is computationally expensive to solve using off-the-shelf algorithms, e.g., the interior point method, because of the constraints, as well as the non-linearity of the objective function. However, as we will show shortly, we can find a closed form optimal solution to this problem, using tools from the theory of doubly stochastic matrices~\citep{dufosse2016notes} and perfect bipartite graph matching. Due to space limit, we defer the detailed derivation and proof to appendix, and only show a sketch below. 

Without loss of generality, for any feasible $\Sigma_1$, using spectral decomposition, we can reparametrize $\Sigma_1$ as
\begin{equation}
\Sigma_1 = U\Lambda U^T,\quad \Lambda = \diag(\lambda_1, \ldots, \lambda_d)
\label{equ:t1}
\end{equation}
where $u \geq \lambda_1\geq\lambda_2\cdots\geq \lambda_d\geq l$. Similarly, we can represent 
\begin{equation}
W\Sigma_2W^T = VNV^T,\quad N = \diag(\nu_1, \ldots, \nu_d)
\label{equ:t2}
\end{equation}
where $0\leq \nu_1\leq\cdots\leq \nu_d$. Let $\lambda = (\lambda_1, \cdots, \lambda_d)^T$ and $\nu = (\nu_1, \cdots, \nu_d)^T$. Set $K = U^TV$ and define $P$ to be the Hadamard product of $K$, i.e., $P = K\circ K$. Since both $U$ and $V^T$ are orthonormal matrices, it immediately follows that $K$ is also an orthonormal matrix. As a result, we have the following two equations hold: 
\begin{align*}
\sum_{j=1}^d P_{ij} = \sum_{j=1}^d K_{ij}^2 = 1, \quad\forall i \in [d]\\
\sum_{i=1}^d P_{ij} = \sum_{i=1}^d K_{ij}^2 = 1, \quad\forall j \in [d]
\end{align*}
which implies that $P$ is a doubly stochastic matrix. Given $U$ being an orthonormal matrix, we have $\log|\Sigma_1| = \log|U\Lambda U^T| = \log|\Lambda|$. On the other hand, it can be readily verified that the following equality holds:
\begin{equation}
\tr(\Lambda K N K^T) = \sum_{i=1}^d\sum_{j=1}^d\lambda_i K_{ij}^2 \nu_j = \lambda^T P \nu
\label{equ:t3}
\end{equation}
By combining all the transformations in \eqref{equ:t1}, \eqref{equ:t2} and \eqref{equ:t3} and plug them in \eqref{equ:optsigma}, we have the following equivalent optimization problem:
\begin{align}
& \text{minimize}_{P, \lambda} && \lambda^T P\nu -m\sum_{i=1}^d \log \lambda_i \nonumber\\
& \text{subject to} && l\mathbf{1}_d\leq \lambda \leq u\mathbf{1}_d
\label{equ:short}
\end{align}
where $\mathbf{1}_d$ denotes a vector of all ones with dimension $d$. To solve (\ref{equ:short}), we make the following key observations: 
\begin{enumerate}[noitemsep,topsep=0pt]
	\item 	The minimization is decomposable in terms of $P$ and $\lambda$. Furthermore, the optimization over $P$ is a linear program (LP).
	\item	For any bounded LP, there exists at least one extreme point that achieves the optimal solution.
	\item 	 The set of $d\times d$ doubly stochastic matrices, denoted as $B_d$, forms a convex polytope, known as the Birkhoff polytope.
	\item 	By the Birkhoff-von Neumann theorem, $B_d$ is the convex hull of the set of permutation matrices, i.e., every extreme point of $B_d$ is a permutation matrix.
\end{enumerate}
Combining all the analysis above, it is clear to see that the optimal solution $P$ must be a permutation matrix. This motivates us to reduce (\ref{equ:short}) to a minimum-weight perfect matching problem on a weighted complete bipartite graph as follows: for any $\lambda, \nu\in\RR_+^d$, we can construct a weighted $d\times d$ bipartite graph $G = (V_\lambda, V_\nu, E;w)$ as follows:
\begin{itemize}[noitemsep,topsep=0pt]
  \item   For each $\lambda_i$, construct a vertex $v_{\lambda_i}\in V_\lambda$, $\forall i$.
  \item   For each $\nu_j$, construct a vertex $v_{\nu_j}\in V_\nu$, $\forall j$.
  \item   For each pair $(v_{\lambda_i}, v_{\nu_j})$, construct an edge $e(v_{\lambda_i}, v_{\nu_j})$ with weight $w(e(v_{\lambda_i}, v_{\nu_j})) = \lambda_i \nu_j$.
\end{itemize}
The following theorem relates the solution of the minimum weight matching to the partial solution of (\ref{equ:short}) w.r.t. $P$:
\begin{theorem}
Let $\lambda = (\lambda_1, \ldots, \lambda_d)$ and $\nu = (\nu_1, \ldots, \nu_d)$ with $\lambda_1\geq\cdots\geq\lambda_d$ and $\nu_1\leq\cdots\leq\nu_d$. The minimum-weight perfect matching on $G$ is the set of edges $\pi^* = \{(v_{\lambda_i}, v_{\nu_i}): 1\leq i\leq d\}$ with the minimum weight $w(\pi^*) = \sum_{i=1}^d \lambda_i\nu_{i}$. Furthermore, it equals $\min_P \lambda^T P \nu$.
\label{thm:combined}
\end{theorem}
\emph{Proof sketch}. The full proof of Theorem~\ref{thm:combined} is deferred to the appendix, and here we only show a sketch of the high-level idea. Basically, given any matching in the graph, if there is an inverse pair (a cross) in the matching, then we can improve the matching by re-matching the inverse pair, as shown in Figure~\ref{fig:invpair}.
\begin{figure}[htb]
\centering
  \includegraphics[width=0.8\linewidth]{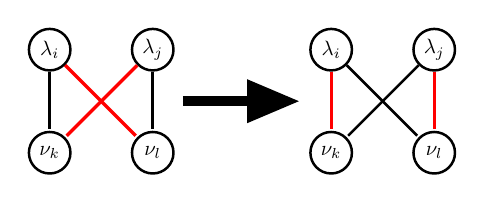}
\caption{Re-matching an inverse pair $(\lambda_i, \lambda_j, \nu_k, \nu_l) = \{(v_{\lambda_i}, v_{\nu_l}), (v_{\lambda_j}, v_{\nu_k})\}$ on the left side to a match with smaller weight $\{(v_{\lambda_i}, v_{\nu_k}), (v_{\lambda_j}, v_{\nu_l})\}$. Red color is used to highlight edges in the perfect matching. }
\label{fig:invpair}
\end{figure}
Now since there are only at most finitely many number of inverse pairs, an inductive argument shows that the optimal matching is achieved when there is no inverse pair, i.e., $v_{\lambda_i}$ is matched to $v_{\nu_i},\forall i\in[d]$ (Figure~\ref{fig:sketch}).
\begin{figure*}[htb]
\centering
	\includegraphics[width=0.9\linewidth]{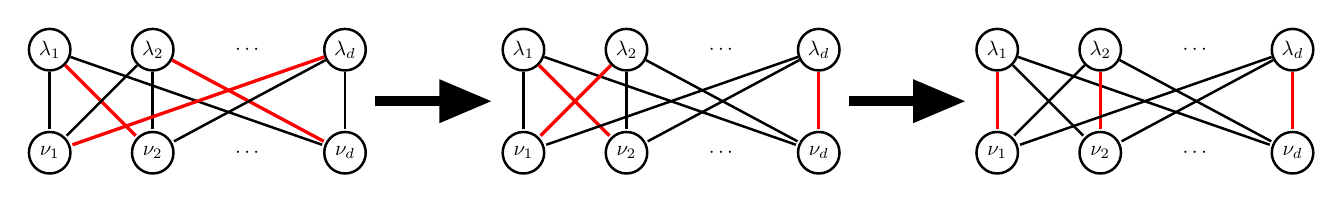}
\caption{The inductive proof works by recursively removing inverse pairs from $(\lambda_d, \nu_d)$ to $(\lambda_1, \nu_1)$. The process stops until there is no inverse pair in the matching. Red color is used to highlight edges in the perfect matching.}
\label{fig:sketch}
\end{figure*}

The optimal matching in Theorem~\ref{thm:combined} suggests that the optimal doubly stochastic matrix is given by $P^* = I_d$, which also implies $K^* = P^* = I_d$ and $U^* = V$. Now plug in the $P^* = I_d$ into (\ref{equ:short}). The optimization w.r.t.\ $\lambda$ decomposes into $d$ independent scalar optimization problems, which can be easily solved. Using the hard-thresholding function defined in \eqref{equ:hardthresholding}, we can express the optimal solution $\lambda_i^*$ as $\lambda_i^* = \mathbb{T}_{[l, u]}(m/\nu_i)$. Combine all the analysis given above, we get the algorithms listed at the beginning of this section to optimize $\Sigma_1$ and $\Sigma_2$. Interestingly, they have close connection to the proximal method proposed in the literature to solve matrix completion~\citep{cai2010singular}, or Euclidean projection under trace norm constraint~\citep{chen2011integrating,chen2012learning}. To the best of our knowledge, this is the first algorithm that solves linear function over matrices with negative log-determinant regularization (e.g.~\eqref{equ:optsigma}) efficiently.  

\section{NONLINEAR EXTENSION}
So far we discuss our FETR framework under the linear regression model, but it can be readily extended to any nonlinear regression/classification settings. One straightforward way to do so is to apply the (orthogonal) random Fourier transformation~\citep{rahimi2008random,felix2016orthogonal} to generate high-dimensional random features so that linear FETR in the transformed space corresponds to nonlinear models in the original feature space. However, depending on the dimension of the random features, this approach might lead to a huge covariance matrix $\Sigma_1$ that is expensive to optimize. 

Another more natural and expressive approach is to combine our regularization scheme and optimization method with parametrized nonlinear feature transformations, such as neural networks. More specifically, let $g(\mathbf{x};\theta):\RR^d\to\RR^p$ be a neural network with learnable parameter $\theta$ that defines a nonlinear transformation of the input features from $\RR^d$ to $\RR^p$. Essentially we can replace the feature matrix $X$ in (\ref{equ:matrixopt}) with $g(\mathbf{x};\theta)$ to create a regularized multitask neural network~\citep{caruana1997multitask} where we add one more layer defined by the matrix $W$ on top of the nonlinear mapping given by $g(\mathbf{x};\theta)$. To train the model, we can use backpropagation to optimize $W$, $\theta$ and our proposed approach to optimize the two covariance matrices. We will further explore this nonlinear extension in Section~\ref{sect:experiments} to demonstrate its power in statistical modeling.
\begin{figure*}[htb]
\begin{subfigure}[t]{0.5\linewidth}
\centering
  \includegraphics[width=\linewidth]{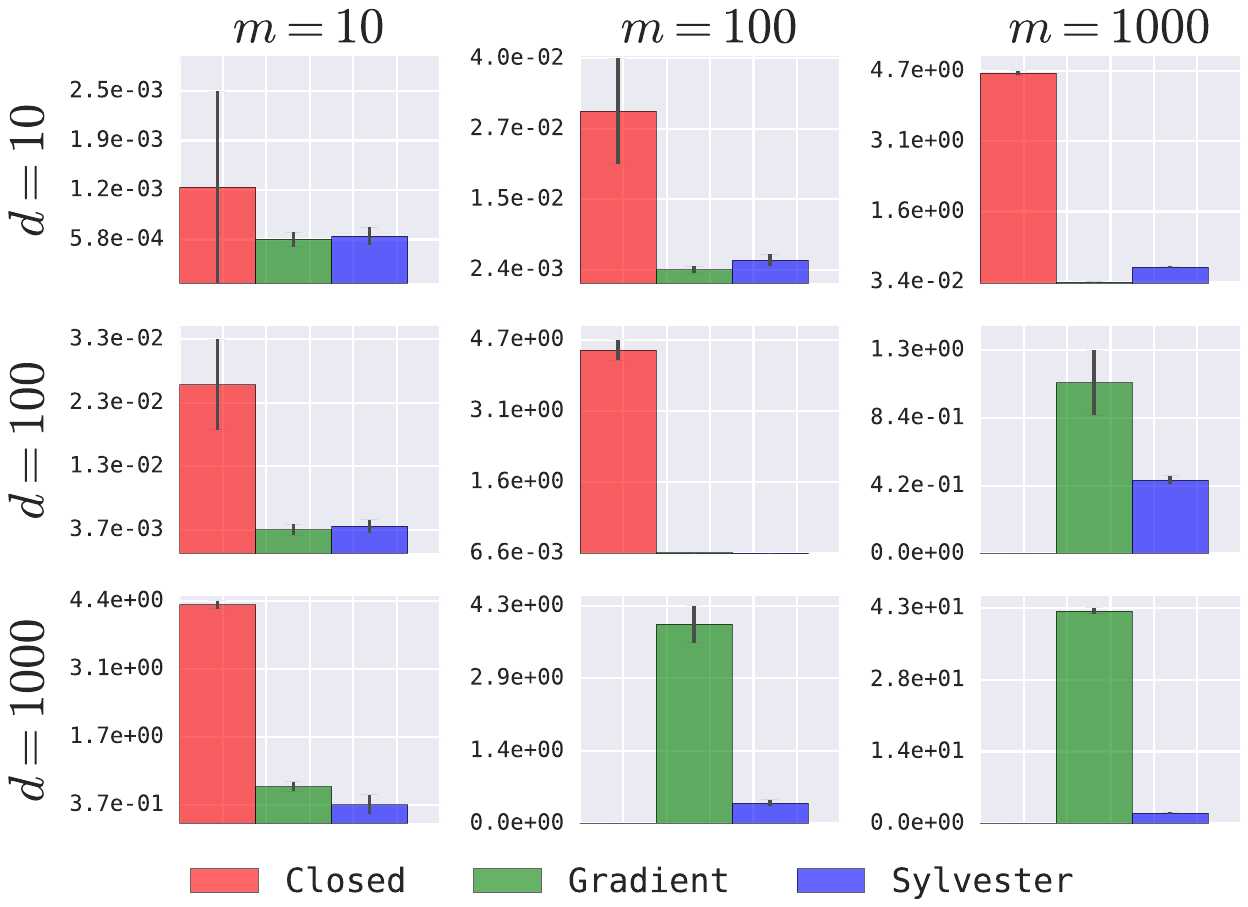}
\caption{The mean run time (seconds) under each experimental configuration. The closed form solution does not scale when $md \geq 10^4$.}
\label{fig:runtime}
\end{subfigure}
~
\begin{subfigure}[t]{0.46\linewidth}
\centering
    \includegraphics[width=\linewidth]{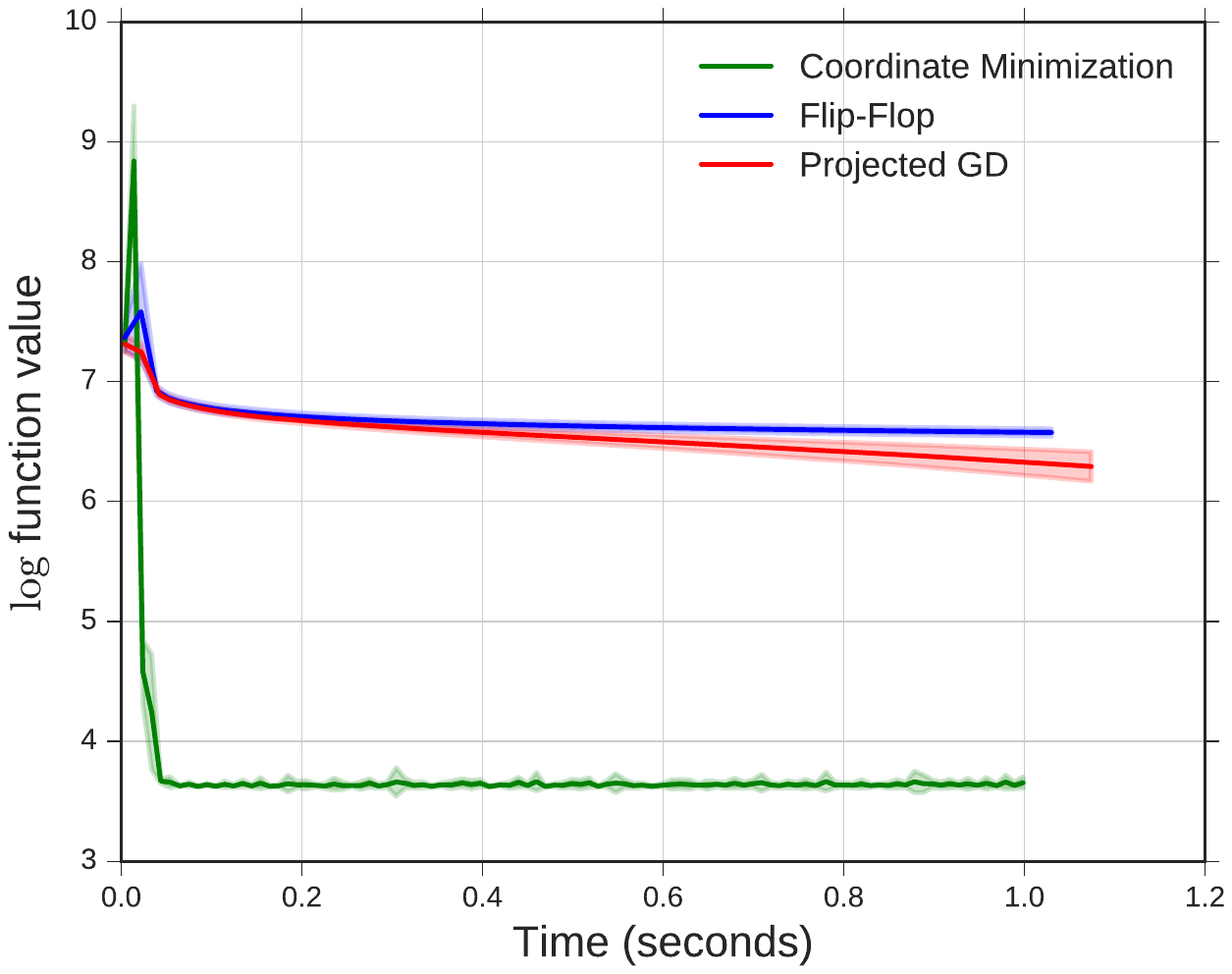}
\caption{The convergence speed of coordinate minimization versus projected gradient descent and the flip-flop algorithm on the SARCOS dataset. All the experiments are repeated 10 times.}
\label{fig:pgd}
\end{subfigure}
\caption{Experimental results of the convergence analysis on synthetic data.}
\end{figure*}

\section{EXPERIMENTS}
\label{sect:experiments}
\subsection{CONVERGENCE ANALYSIS AND COMPUTATIONAL EFFICIENCY}
We first investigate the efficiency and scalability of the three different algorithms for minimizing w.r.t.\ $W$ on synthetic data sets. For each experiment, we generate a synthetic data set which consists of $n = 10^4$ instances that are shared among all the tasks. All the instances are randomly sampled uniformly from $[0, 1]^d$. We gradually increase the dimension of features, $d$, and the number of tasks, $m$ to test scalability. 

The first algorithm implements the closed form solution by explicitly computing the $md\times md$ matrix product and then solving the linear system. The second one is the proposed gradient descent, and the last one uses the Bartels-Stewart algorithm to solve the equivalent Sylvester equation to compute $W$. We use open source toolkit \texttt{scipy} whose backend implementation uses highly optimized Fortran code. For all the synthetic experiments we set $l = 0.01$ and $u = 100$, which corresponds to a condition number of $10^4$. We fix the coefficients $\eta = 1.0$. We repeat each experiment for 10 times to show both the mean and the variance. 

The experimental results are shown in Figure~\ref{fig:runtime}. As expected, the closed form solution does not scale to problems of even moderate size due to its large memory requirement. In practice the Bartels-Stewart algorithm is about one order of magnitude faster than the gradient descent method when either $m$ or $d$ is large. It is also worth pointing out here that the Bartels-Stewart algorithm is the most numerically stable algorithm among the three based on our observations. 

\begin{table*}[htb]
\caption{Mean squared error on the SARCOS data and the mean of normalized mean squared error (NMSE) on the school dataset across 10-fold cross-validation.}
\label{table:sarcos}
% \vspace{-1em}
\begin{center}
\begin{small}
\begin{sc}
\begin{tabular}{l*7r|r}
\toprule
\multicolumn{1}{c}{\multirow{3}{*}{Method}} & \multicolumn{8}{c}{Datasets} \\
\cmidrule{2-9} 
 & \multicolumn{7}{c|}{Sarcos}     & \multicolumn{1}{c}{\multirow{2}{*}{School}} \\
 & \multicolumn{1}{c}{1st}  & \multicolumn{1}{c}{2nd}  & \multicolumn{1}{c}{3rd}  & \multicolumn{1}{c}{4th}  & \multicolumn{1}{c}{5th}  & \multicolumn{1}{c}{6th}  & \multicolumn{1}{c|}{7th}  &  \\
\midrule
\texttt{STL} & 31.40 & 22.90  & 9.13   & 10.30 &  0.14 & 0.84 & 0.46 & 0.9882 $\pm$ 0.0196\\
\midrule
\texttt{MTFL} & 31.41 & 22.91 & 9.13 & 10.33  & 0.14 & \textbf{0.83} & 0.45 & 0.8891 $\pm$ 0.0380\\
\texttt{MTRL} & 31.09 & 22.69 & \textbf{9.08} & 9.74  & 0.14 & \textbf{0.83} & 0.44 & 0.9007 $\pm$ 0.0407 \\
\texttt{MTFRL} &  31.13 & \textbf{22.60} & 9.10 & 9.74 & \textbf{0.13} & \textbf{0.83} & 0.45 & 0.8451 $\pm$ 0.0197\\
\midrule
\texttt{FETR} & \textbf{31.08} & 22.68 & \textbf{9.08} & \textbf{9.73} & \textbf{0.13} & \textbf{0.83} & \textbf{0.43} & \textbf{0.8134 $\pm$ 0.0253}\\
\midrule
\texttt{STL-NN} & 24.81 & 17.20 & 8.97 & 8.36 & 0.13 & 0.72 & 0.34 & \multicolumn{1}{c}{$-$}\\
\texttt{MT-NN} & 12.01	& 10.54 & 5.02 & 7.15 & 0.09 & 0.70 & 0.27 & \multicolumn{1}{c}{$-$}\\
\texttt{MTFRL-NN} & 11.02 & 9.51 & 4.99 & 7.11 & \textbf{0.08} & 0.62 & 0.27 & \multicolumn{1}{c}{$-$} \\
\texttt{FETR-NN} & \textbf{10.77} & \textbf{9.34} & \textbf{4.95} & \textbf{7.01} & \textbf{0.08} & \textbf{0.59} & \textbf{0.24} & \multicolumn{1}{c}{$-$} \\
\bottomrule
\end{tabular}
\end{sc}
\end{small}
\end{center}
\vspace{-1em}
\end{table*}

We compare our proposed coordinate minimization algorithm with an off-the-shelf projected gradient method and the flip-flop algorithm to solve the optimization problem (\ref{equ:matrixopt}). Specifically, the projected gradient method updates $W, \Sigma_1$ and $\Sigma_2$ in each iteration and then projects $\Sigma_1$ and $\Sigma_2$ onto the corresponding feasible regions. The flip-flop algorithm is implemented as suggested in \citet{zhang2010learning} and we use a fudge factor of $10^{-3}$ to avoid the nonpositive definite problem. In each iteration, both covariance matrices are projected onto the feasible region as well. In the SARCOS dataset all the instances are shared among all the tasks, so that the Sylvester solver is used to optimize $W$ in coordinate minimization. We repeat the experiments 10 times and report the mean and standard deviation of the $\log$ function values versus the time used by all three algorithms (Figure~\ref{fig:pgd}). It is clear from Figure~\ref{fig:pgd} that our proposed algorithm not only converges much faster than the other two competitors, but also achieves better results. In fact, as we observe in our experiments, the proposed algorithm usually converges in less than 10 iterations. 

\subsection{REAL-WORLD DATASETS}
\vspace{-0.6em}
In this section we apply FETR to two real-world datasets to demonstrate its statistical efficiency. 

\textbf{Robot Inverse Dynamics} This data relates to an inverse dynamics problem for a seven degree-of-freedom (DOF) SARCOS anthropomorphic robot arm~\citep{vijayakumar2000locally}. The goal is to map from a 21-dimensional input space (7 joint positions, 7 joint velocities, 7 joint accelerations) to the corresponding 7 joint torques. Hence there are 7 tasks and the inputs are shared among all the tasks. The training set and test set contain 44,484 and 4,449 examples, respectively. We further partition the training set into a training set and a validation set, containing 31,138 and 13,346 instances, respectively.

\textbf{School Data} This dataset consists of the examination scores of 15,362 students from 139 secondary schools~\citep{goldstein1991multilevel}. It has 27 input features, and contains 139 tasks. Since the train/test splits are not provided, we use a 10-fold cross-validation procedure to generate the training and test datasets.

\subsection{SETUP AND RESULTS}
\vspace{-0.6em}
We compare \texttt{FETR} with multitask feature learning~\citep{evgeniou2007multi} (\texttt{MTFL}), multitask relationship learning~\citep{zhang2010convex} (\texttt{MTRL}), and the \texttt{MTFRL} framework. 
We also use ridge regression as our baseline model, denoted as single task learning (\texttt{STL}).
The results reported for \texttt{FETR} on the SARCOS dataset are obtained using the Sylvester equation solver, while for the School dataset the inputs are not shared among different tasks and hence we use our gradient descent solver for $W$ instead.
To evaluate, we compute the mean of normalized mean squared error (NMSE) over the output tasks (e.g., 139 tasks for the School data).
The NMSE is defined as the ratio of the MSE and the variance on a task.
For the School dataset, we show the mean NMSE and its standard deviation across 10 cross-validation folds, since no train/test splits are provided.

\begin{figure}[t!]
\centering
\vspace{-1em}
\begin{subfigure}[b]{\linewidth}
\centering
  \includegraphics[width=0.8\linewidth]{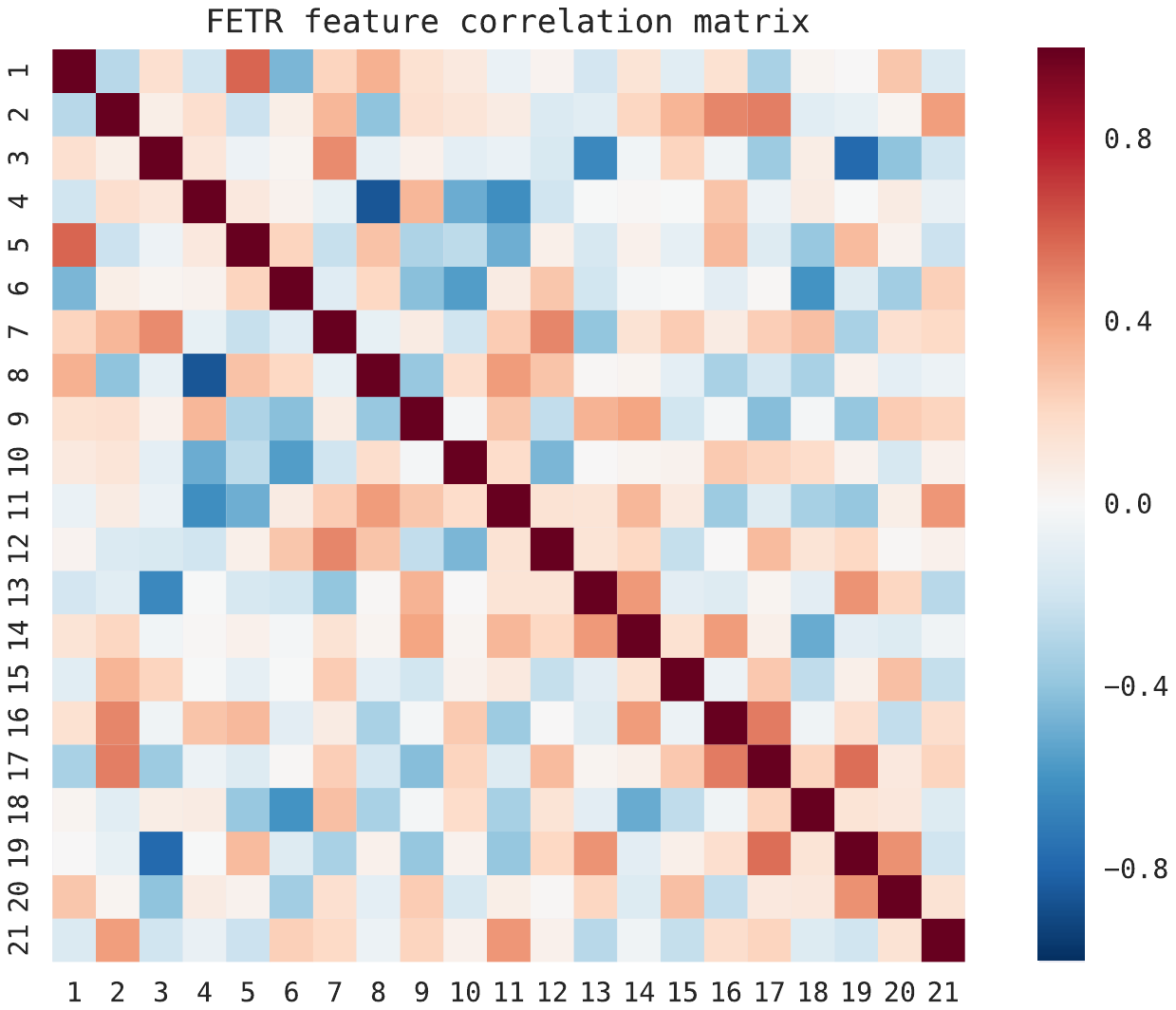}
\caption{Feature covariance matrix.}
\end{subfigure}
~
\begin{subfigure}[b]{\linewidth}
\centering
    \includegraphics[width=0.8\linewidth]{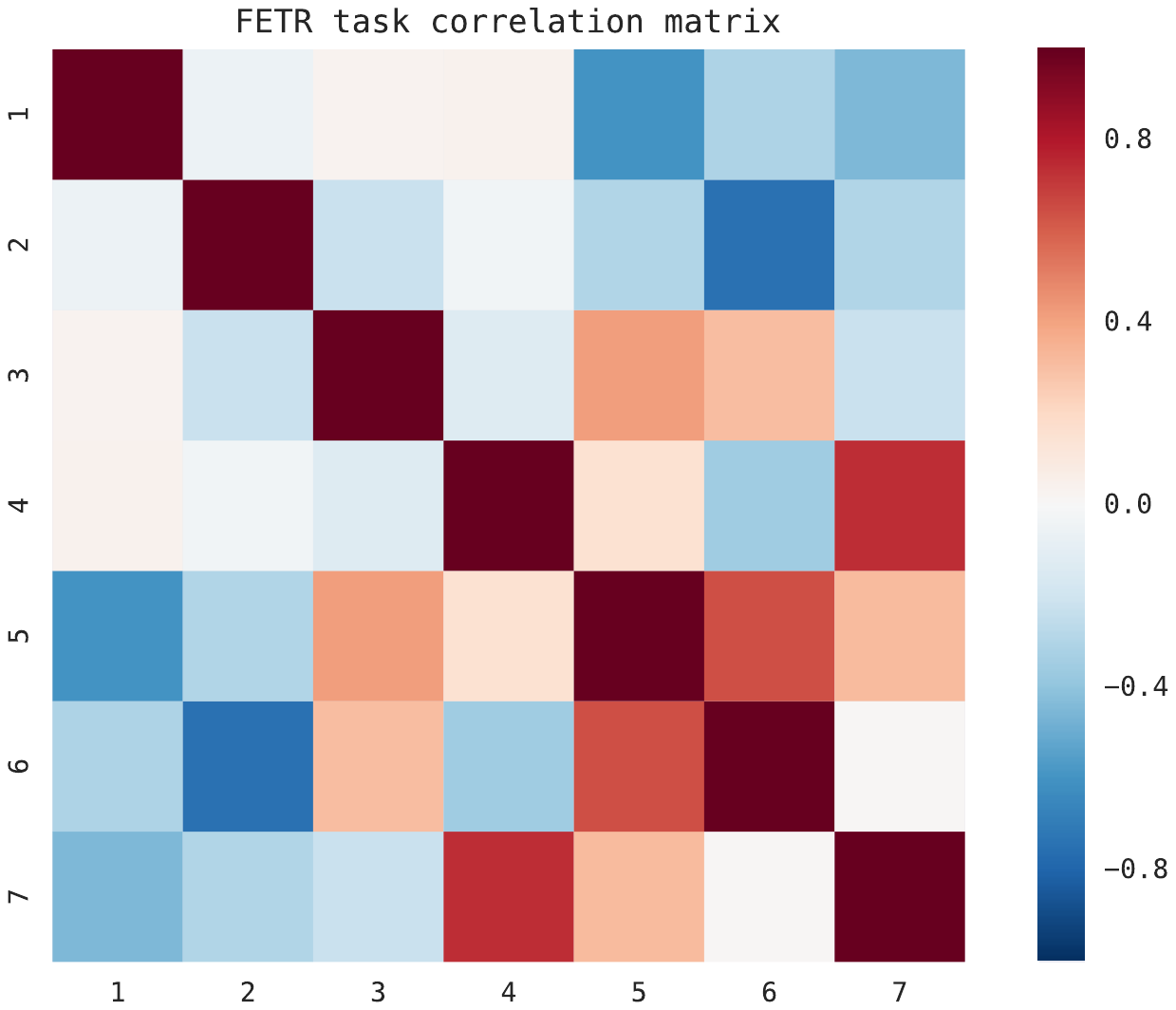}
\caption{Task covariance matrix.}
\end{subfigure}
\caption{Estimated feature and task covariance matrices on the SARCOS dataset.}
\label{fig:covariance}
\vspace*{-1.5em}
\end{figure}

To show the power of the nonlinear extension, we also run experiments using a single task neural network (\texttt{STL-NN}) and the multitask neural network (\texttt{MT-NN}), based on which we propose our \texttt{FETR-NN}, which incorporates the regularization scheme into the last layer of the \texttt{MT-NN}. \texttt{STL-NN} is a model where we use a separate network for each task, while in \texttt{MT-NN} all layers except the last output layer are shared among different tasks. As another baseline, we also compare our method with the multilinear relationship network~\citep{long2017learning}, which can be understood as an extension of the MTFRL method using neural networks (\texttt{MTFRL-NN}). 
In all experiments, \texttt{STL-NN}, \texttt{MT-NN}, \texttt{MTFRL-NN} and \texttt{FETR-NN} share exactly the same network structure: an input layer with 21 dimensions, followed by two hidden layers with 256 and 100 hidden units. The output of the network in \texttt{MT-NN}, \texttt{MTFRL-NN} and \texttt{FETR-NN} is a multitask layer that contains 7 output units, while in \texttt{STL-NN}, the output only contains a single unit. All the methods share the same experimental setting, including model selection. In all the experiments we fix $l = 10^{-3}$ and $u = 10^3$. The hyperparameters range from $\eta \in \{10^{-5}, \ldots, 10^3\}$, and we use the validation set for model selection. Note that because the instances are not shared between different tasks for the School dataset, \texttt{MT-NN}, \texttt{MTFRL-NN} and \texttt{FETR-NN} cannot be directly applied. For each method, the best model on the validation set is selected.

The results are summarized in Table~\ref{table:sarcos} (the smaller the better). Among all the methods, \texttt{FETR} consistently achieves lower test set NMSEs. Moreover, we observe a significant improvement of both \texttt{MT-NN}, \texttt{MTFRL-NN} and \texttt{FETR-NN} over all the linear baselines and \texttt{STL-NN}. \texttt{FETR-NN} further improves over \texttt{MT-NN} and \texttt{MTFRL-NN} on all the tasks. The experimental results confirm that multitask learning usually improves over single task learning when the dataset is small and tasks are related. Furthermore, among all the competitors, we observe that nonlinear models combined with our \texttt{FETR} framework give the overall best results, demonstrating the effectiveness of the proposed approach in both linear and nonlinear settings. 

One by-product of FETR is that we also have access to the estimated row and column covariance matrices. In Figure~\ref{fig:covariance} we plot the feature and task covariance matrices respectively, where we can clearly observe a block diagonal structure: the first 4 tasks are negatively correlated with the rest 3, and the 5th and 6th task are positively correlated. Intuitively, these correlations are consistent with the SARCOS dataset where several joints move jointly. 

% \vspace{-0.7em}
\section{CONCLUSIONS}
\vspace{-0.7em}
In this paper we point out a common flaw in the existing multitask feature and relationship learning frameworks, and propose a constrained variant to fix it. Our framework admits a multiconvex formulation, which allows us to design an efficient block coordinate-wise algorithm to optimize. To solve the weight learning subproblem, we propose three different strategies that can be used no matter whether the instances are shared by multiple tasks or not. To learn the covariance matrices, we reduce the underlying matrix optimization subproblem to a minimum weight perfect matching problem, and solve it exactly in closed form. To the best of our knowledge, all the previous methods have to resort to expensive iterative procedures to solve this problem. At the end, we also discuss several possible extensions of the proposed framework to nonlinear settings. Experimental results show that our method is orders of magnitude faster than its competitors, and it demonstrates significantly improved statistical performance on two real-world datasets. 

\subsubsection*{Acknowledgements}
\vspace{-0.5em}
HZ and GG would like to acknowledge the support from DARPA XAI project, contract FA87501720152.
OS is supported by a CMLH Fellowship in Digital Health and by NIH under grant U01NS098969.

\bibliography{reference}
\bibliographystyle{abbrvnat}

\newpage
\onecolumn
\appendix 
\section*{Supplementary Material of ``Efficient Multitask Feature and Relationship Learning''}
In this supplementary material we provide all the missing proofs to the propositions, lemmas and theorems in our main paper ``Efficient Multitask Feature and Relationship Learning''. 

\section{PROOFS}
\subsection{Proof of Proposition~\ref{claim:closed}}
\closed*
To prove this claim, we need the following facts about tensor product:
\begin{fact}
Let $A$ be a matrix. Then $||A||_F = ||\text{vec}(A)||_2$.
\label{fact:3.1}
\end{fact}
\begin{fact}
Let $A\in\RR^{m_1\times n_1}$, $B\in\RR^{n_1\times n_2}$ and $C\in \RR^{n_2\times m_2}$. Then $\text{vec}(ABC) = (C^T\otimes A)\text{vec}(B)$.
\label{fact:3.2}
\end{fact}
\begin{fact}
Let $S_1\in\RR^{m_1\times n_1}, S_2\in\RR^{n_1\times p_1}$ and $T_1\in\RR^{m_2\times n_2}, T_2\in\RR^{n_2\times p_2}$. Then $(S_1\otimes S_2)(T_1\otimes T_2) = (S_1S_2)\otimes (T_1T_2)$.
\label{fact:3.3}
\end{fact}
\begin{fact}
Let $A\in\RR^{n\times n}$ and $B\in\RR^{m\times m}$. Let $\{\mu_1, \ldots, \mu_n\}$ be the spectrum of $A$ and $\{\nu_1, \ldots, \nu_m\}$ be the spectrum of $B$. Then the spectrum of $A\otimes B$ is $\{\mu_i\nu_j: 1\leq i \leq n, 1\leq j \leq m\}$.
\label{fact:3.4}
\end{fact}
We can show the following result by transforming $W$ into its isomorphic counterpart:
\begin{proof}
\begin{align*}
 &\phantom{{}={}} ||Y-XW||_F^2 + \eta~||\Sigma_1^{1/2}W\Sigma_2^{1/2}||_F^2 \\
&=  ||\text{vec}(Y-XW)||_2^2 + \eta~||\text{vec}(\Sigma_1^{1/2}W\Sigma_2^{1/2})||_2^2 && \text{(By Fact~\ref{fact:3.1})}\\
&=  ||\text{vec}(Y) - (I_m\otimes X)\text{vec}(W)||_2^2 + \eta~||(\Sigma_2^{1/2}\otimes \Sigma_1^{1/2})\text{vec}(W)||_2^2 && \text{(By Fact~\ref{fact:3.2})} \\
& = \text{vec}(W)^T\left((I_m\otimes X)^T(I_m\otimes X) + \eta (\Sigma_2^{1/2}\otimes \Sigma_1^{1/2})^T(\Sigma_2^{1/2}\otimes \Sigma_1^{1/2})\right)\text{vec}(W) \\
&\phantom{{}={}} - 2\text{vec}(W)^T(I_m\otimes X^T)\text{vec}(Y) + \text{vec}(Y)^T\text{vec}(Y) \\
& = \text{vec}(W)^T\left((I_m\otimes X^TX) + \eta (\Sigma_2\otimes \Sigma_1)\right)\text{vec}(W) \\
&\phantom{{}={}}- 2\text{vec}(W)^T(I_m\otimes X^T)\text{vec}(Y) + \text{vec}(Y)^T\text{vec}(Y) && \text{(By Fact~\ref{fact:3.3})}
\end{align*}
The last equation above is a quadratic function of $\text{vec}(W)$, from which we can read off that the optimal solution $W^*$ should satisfy:
\begin{equation}
\text{vec}(W^*) = \left(I_m\otimes (X^TX) + \eta\Sigma_2\otimes \Sigma_1\right)^{-1}\text{vec}(X^TY)
\end{equation}
$W^*$ can then be obtained simply by reformatting $\text{vec}(W^*)$ into a $d\times m$ matrix. The computational bottleneck in the above procedure is in solving an $md\times md$ system of equations, which scales as $O(m^3d^3)$ if no further structure is available. The overall computational complexity is $O(m^3d^3 + mnd^2)$. 
\end{proof}

\subsection{Proof of Proposition~\ref{thm:wconverge}}
To analyze the convergence rate of gradient descent in this case, we start by bounding the smallest and largest eigenvalue of the quadratic system. 
\begin{lemma}[Weyl's inequality] 
Let $A, B$ and $C$ be $n$-by-$n$ Hermitian matrices, and $C = A+B$. Let $a_1\geq\cdots\geq a_n$, $b_1\geq\cdots\geq b_n$ and $c_1\geq\cdots\geq c_n$ be the eigenvalues of $A, B$ and $C$ respectively. Then the following inequalities hold for $r + s - 1 \leq i \leq j+k-n$, $\forall i = 1, \ldots, n$:
$$a_j + b_k \leq c_i \leq a_r + b_s$$
\end{lemma}
Let $\lambda_k(A)$ be the $k$-th largest eigenvalue of matrix $A$.
\begin{lemma}
If $\Sigma_1$ and $\Sigma_2$ are feasible in (\ref{equ:matrixopt}), then
$$\lambda_1(I_m\otimes (X^TX) + \eta I_{md} + \rho\Sigma_2\otimes \Sigma_1)\leq \lambda_1(X^TX) + \eta + \rho u^2$$
$$\lambda_{md}(I_m\otimes (X^TX) + \eta I_{md} + \rho\Sigma_2\otimes\Sigma_1)\geq \lambda_d(X^TX) + \eta + \rho l^2$$
\end{lemma}
\begin{proof}
By Weyl's inequality, setting $r = s = i = 1$, we have $c_1 \leq a_1 + b_1$. Set $j = k = i = n$, we have $c_n \geq a_n + b_n$. We can bound the largest and smallest eigenvalues of $I_m\otimes (X^TX) + \eta I_{md} + \rho\Sigma_2\otimes\Sigma_1$ as follows:
\begin{align*}
&\phantom{{}={}} \lambda_1(I_m\otimes (X^TX) + \eta I_{md} + \rho\Sigma_2\otimes \Sigma_1) \\
 &\leq \lambda_1(I_m\otimes (X^TX)) + \lambda_1(\eta I_{md}) + \lambda_1(\rho\Sigma_2\otimes\Sigma_1) && \text{(By Weyl's inequality)}\\
&=  \lambda_1(I_m)\lambda_1(X^TX) + \eta + \rho\lambda_1(\Sigma_1)\lambda_1(\Sigma_2) && \text{(By Fact~\ref{fact:3.4})}\\
&\leq  \lambda_1(X^TX) + \eta + \rho u^2 && \text{(By the feasibility assumption)}
\end{align*}
and 
\begin{align*}
&\phantom{{}={}} \lambda_{md}(I_m\otimes (X^TX) + \eta I_{md} + \rho\Sigma_2\otimes \Sigma_1) \\
&\geq  \lambda_{md}(I_m\otimes (X^TX)) + \lambda_{md}(\eta I_{md}) + \lambda_{md}(\rho\Sigma_2\otimes\Sigma_1)&& \text{(By Weyl's inequality)}\\
&=  \lambda_m(I_m)\lambda_d(X^TX) + \eta + \rho\lambda_m(\Sigma_1)\lambda_d(\Sigma_2) && \text{(By Fact~\ref{fact:3.4})}\\
&\geq  \lambda_d(X^TX) + \eta + \rho l^2 && \text{(By the feasibility assumption)}
\end{align*}
\end{proof}
We will first introduce the following two lemmas adapted from~\citep{nesterov2013introductory}, using the fact that the spectral norm of the Hessian matrix $\nabla^2 h(W)$ is bounded. 
\begin{lemma}
Let $f(W): \RR^{d\times m}\mapsto\RR$ be a twice differentiable function with $\lambda_1(\nabla^2 f(W))\leq L$. $L > 0$ is a constant. The minimum value of $f(W)$ can be achieved. Let $W^* = \argmin_W f(W)$, then 
$$f(W^*)\leq f(W) - \frac{1}{2L}||\nabla f(W)||_F^2$$ 
\label{lemma:1}
\end{lemma}
%\begin{proof}
%Since $f(W)$ is twice differentiable with $\lambda_1(\nabla^2 f(W))\leq L$, by the Lagrangian mean value theorem, $\forall W, \widetilde{W}$, we can find a value $0 < t(W, \widetilde{W}) < 1$, such that
%\begin{align*}
%f(\widetilde{W}) &= f(W) + \tr(\nabla f(W)^T(\widetilde{W} - W)) + \frac{1}{2}\text{vec}(\widetilde{W} - W)^T\nabla^2 f(tW + (1-t)\widetilde{W}) \text{vec}(\widetilde{W} - W) \\
%& \leq f(W) + \tr(\nabla f(W)^T(\widetilde{W} - W)) + \frac{L}{2}||\widetilde{W} - W||_F^2
%\end{align*}
%Since $W^*$ achieves the minimum value of $f(W)$, we can use the above result to obtain:
%\begin{align*}
%f(W^*) &= \inf_{\widetilde{W}} f(\widetilde{W}) \\
%&\leq \inf_{\widetilde{W}}f(W) + \tr(\nabla f(W)^T(\widetilde{W} - W)) + \frac{L}{2}||\widetilde{W} - W||_F^2\\
%& = f(W) - \frac{1}{2L}||\nabla f(W)||_F^2
%\end{align*}
%where the last equation comes from the fact that the minimum of a quadratic function with respect to $\widetilde{W}$ can be achieved at $\widetilde{W} = W - \frac{1}{L}\nabla f(W)$. 
%\end{proof}

\begin{lemma}
\label{lemma:2}
Let $f(W): \RR^{d\times m}\mapsto\RR$ be a convex, twice differentiable function with $\lambda_1(\nabla^2 f(W))\leq L$. $L > 0$ is a constant, then $\forall W_1, W_2$:
$$\tr\left((\nabla f(W_1) - \nabla f(W_2))^T(W_1 - W_2)\right) \geq \frac{1}{L}||\nabla f(W_1) - \nabla f(W_2)||_F^2$$
\end{lemma}
%\begin{proof}
%For all $W_1, W_2$, we can construct the following two functions:
%$$f_{W_1}(Z) = f(Z) - \tr\left(\nabla f(W_1)^TZ\right),\qquad f_{W_2}(Z) = f(Z) - \tr\left(\nabla f(W_2)^TZ\right)$$
%Since $f(W)$ is a convex, twice differentiable function with respect to $W$, it follows that both $f_{W_1}(Z)$ and $f_{W_2}(Z)$ are convex, twice differentiable functions with respect to $Z$. The first-order optimality condition of convex functions gives the following conditions to hold for $Z$ which achieves the optimality:
%$$\nabla f_{W_1}(Z) = \nabla f(Z) - \nabla f(W_1) = 0,\qquad \nabla f_{W_2}(Z) = \nabla f(Z) -  f(W_2) = 0$$
%Plug in $W_1$ and $W_2$ into the above optimality conditions respectively. From the first-order optimality condition we know that $W_1$ and $W_2$ achieves the optimal solutions of $f_{W_1}(Z)$ and $f_{W_2}(Z)$, respectively. 
%
%Now applying Lemma~\ref{lemma:1} to $f_{W_1}(Z)$ and $f_{W_2}(Z)$, we have:
%$$\left((W_2) - \tr\left(\nabla f(W_1)^T W_2\right)\right) - \left((W_1) - \tr\left(\nabla f(W_1)^T W_1\right)\right)\geq \frac{1}{2L}||\nabla f(W_1) - \nabla f(W_2)||_F^2$$
%$$\left((W_1) - \tr\left(\nabla f(W_2)^T W_1\right)\right) - \left((W_2) - \tr\left(\nabla f(W_2)^T W_2\right)\right)\geq \frac{1}{2L}||\nabla f(W_1) - \nabla f(W_2)||_F^2$$
%Adding the above two equations leads to
%$$\tr\left((\nabla f(W_1) - \nabla f(W_2))^T(W_1 - W_2)\right)\geq \frac{1}{L}||\nabla f(W_1) - \nabla f(W_2)||_F^2$$
%\end{proof}
We can now proceed to show Proposition~\ref{thm:wconverge}.
\gdconverge*
\begin{proof}
Define function $g(W)$ as follows:
$$g(W) = h(W) - \frac{\lambda_{l}}{2}||W||_F^2$$
Since we have already bounded that $\lambda_{md}(\nabla^2 h(W))\geq \lambda_l$, it follows that $g(W)$ is a convex function and furthermore $\lambda_1(\nabla^2 g(W)) \leq \lambda_u- \lambda_l$. Applying Lemma~\ref{lemma:2} to $g$, $\forall W_1, W_2\in\RR^{d\times m}$, we have:
$$\tr\left((\nabla g(W_1) - \nabla g(W_2))^T(W_1 - W_2)\right)\geq \frac{1}{\lambda_u - \lambda_l}||\nabla g(W_1) - \nabla g(W_2)||_F^2$$
Plug in $\nabla g(W) = \nabla h(W) - \lambda_l W$ into the above inequality and after some algebraic manipulations, we have:
\begin{equation}
\tr\left((\nabla h(W_1) - \nabla h(W_2))^T(W_1 - W_2)\right)\geq \frac{1}{\lambda_u + \lambda_l}||\nabla h(W_1) - \nabla h(W_2)||_F^2 + \frac{\lambda_u \lambda_l}{\lambda_u + \lambda_l}||W_1 - W_2||_F^2
\label{equ:bound}
\end{equation}
Let $W^* = \argmin_W h(W)$. Within each iteration of the algorithm, we have the update formula as $W^+ = W - t\nabla h(W)$, we can bound $||W^+ - W^*||_F^2$ as follows
\begin{align*}
||W^+ - W^*||_F^2 &= ||W - W^* - t\nabla h(W)||_F^2 \\
& = ||W - W^*||_F^2 + t^2 ||\nabla h(W)||_F^2 - 2t\tr\left((W - W^*)^T\nabla h(W)\right) \\
&\leq (1 - 2t\frac{\lambda_u\lambda_l}{\lambda_u + \lambda_l})||W - W^*||_F^2 + t(t - \frac{2}{\lambda_u + \lambda_l})||\nabla h(W)||_F^2 && \text{(By inequality~\ref{equ:bound})}\\
&\leq (1 - 2t\frac{\lambda_u\lambda_l}{\lambda_u + \lambda_l})||W - W^*||_F^2 && \text{(For }0 < t\leq 2/(\lambda_u + \lambda_l)\text{)}
\end{align*}
Apply the above inequality recursively for $T$ times, we have
$$||W^{(T)} - W^*||_F^2 \leq \gamma^T ||W^{(0)} - W^*||_F^2$$
where $\gamma = 1 - 2t\frac{\lambda_u\lambda_l}{\lambda_u + \lambda_l}$. For $t = 2 / (\lambda_u + \lambda_l)$, we have 
$$\gamma = 1 - 4\lambda_u\lambda_u / (\lambda_l + \lambda_u)^2 = \left(\frac{\lambda_u - \lambda_l}{\lambda_u + \lambda_l}\right)^2$$
Now pick $\forall \varepsilon > 0$, setting the upper bound $\gamma^T ||W^{(0)} - W^*||_F^2 \leq \varepsilon$ and solve for $T$, we have
$$T\geq \log_{1/\gamma}(C/\varepsilon) = O(\log_{1/\gamma}(1/\varepsilon)) = O(\kappa\log(1/\varepsilon))$$
where $C = ||W^{(0)} - W^*||_F^2$ is a constant, and $\kappa = \lambda_u/\lambda_l$ is the condition number. 
\end{proof}

\subsection{Detailed Proof on Optimization of $\Sigma_1$ and $\Sigma_2$}
In this section we show the detailed derivation on how to solve~\eqref{equ:optsigma} efficiently.

As mentioned in the main paper, since $\Sigma_2\in\PD^m$, it follows that $W\Sigma_2W^T\in\SPD^d$. Without loss of generality, using spectral decomposition, we can reparametrize $\Sigma_1 = U\Lambda U^T$, where $\Lambda = \diag(\lambda_1, \ldots, \lambda_d)$ with $u \geq \lambda_1\geq\lambda_2\cdots\geq \lambda_d\geq l$ and $U\in\RR^{d\times d}$ with $U^TU = UU^T = I_d$. Similarly, we can represent $W\Sigma_2W^T = VNV^T$ where $V\in\RR^{d\times d}$, $V^TV = VV^T = I_d$ and $N = \diag(\nu_1, \ldots, \nu_d)$ with $0\leq \nu_1\leq\cdots\leq \nu_d$. Note that the eigenvectors in $N$ corresponds to eigenvalues in increasing order rather than decreasing order, for reasons that will become clear below. Realizing that $U$ is an orthonormal matrix, we have:
\begin{equation}
\log|\Sigma_1| = \log|U\Lambda U^T| =  \log|\Lambda|,\quad \tr(\Sigma_1W\Sigma_2 W^T) = \tr(\Lambda U^T V NV^TU)
\label{equ:first}
\end{equation}

Set $K = U^TV$. Since both $U$ and $V$ are orthonormal matrices, $K$ is also an orthonormal matrix. We can further transform (\ref{equ:first}) to be $\tr(\Lambda U^T V NV^TU) = \tr((\Lambda K)(KN)^T)$. Note that the mapping between $U$ and $K$ is bijective since $V$ is a fixed orthonormal matrix. Using $K$ and $\Lambda$, we can equivalently transform the optimization problem (\ref{equ:optsigma}) into the following new form:
\begin{equation}
\underset{K, \Lambda}{\min}~\tr((\Lambda K)(KN)^T) - m\log|\Lambda|, \quad
\text{s.t.}~l\,\diag(\mathbf{1}_d)\leq \Lambda \leq u\,\diag(\mathbf{1}_d), K^TK = KK^T = I_d 
\end{equation}
where $\mathbf{1}_d$ is a $d$-dimensional vector of all ones. At first glance it seems that the new form of optimization is more complicated to solve since it is even not a convex problem due to the quadratic equality constraint. However, as we will see shortly, the new form helps to decouple the interaction between $K$ and $\Lambda$ in that $K$ does not influence the second term $-m\log|\Lambda|$. This implies that we can first partially optimize over $K$, finding the optimal solution as a function of $\Lambda$, and then optimize over $\Lambda$. Mathematically, it means:
\begin{equation}
\min_{K, \Lambda} -m\log|\Lambda| + \tr((\Lambda K)(KN)^T) \quad\Leftrightarrow\quad \min_{\Lambda}-m\log|\Lambda| + \min_{K}\tr((\Lambda K)(KN)^T)
\label{equ:partial}
\end{equation}
So that we can first consider the minimization over $K$:
$\tr((\Lambda K)(KN)^T) = \sum_{i=1}^d \sum_{j=1}^d \lambda_i K^2_{ij}\nu_j = \lambda^T P\nu$, where we define $P = K\circ K$, $\lambda = (\lambda_1, \cdots, \lambda_d)^T$ and $\nu = (\nu_1, \cdots, \nu_d)^T$. Since $K$ is an orthonormal matrix, we have the following two equations: $\sum_{j=1}^d P_{ij} = \sum_{j=1}^d K_{ij}^2 = 1, \quad\forall i \in [d]$, $\sum_{i=1}^d P_{ij} = \sum_{i=1}^d K_{ij}^2 = 1, \quad\forall j \in [d]$, which implies that $P$ is a doubly stochastic matrix%. The partial minimization over $K$ can be equivalently solved by the partial minimization over $P$:
\begin{equation}
\min_{K} \tr((\Lambda K)(KN)^T) = \min_{P}\lambda^T P\nu
\label{equ:doubly}
\end{equation}
In order to solve the minimization over the doubly stochastic matrix $P$, we need to introduce the following theorems.
\begin{lemma}[\citep{bertsimas1997introduction}]
Consider the minimization of a linear program over a polyhedron $P$. Suppose that $P$ has at least one extreme point and that there exists an optimal solution. Then there exists an optimal solution that is an extreme point of $P$. 
\label{thm:lp}
\end{lemma} 
\begin{definition}[Birkhoff polytope]
The \emph{Birkhoff polytope} $B_d$ is the set of $d\times d$ doubly stochastic matrices. $B_d$ is a convex polytope. 
\end{definition}
\begin{lemma}[Birkhoff-von Neumann theorem]
Let $B_d$ be the Birkhoff polytope. $B_d$ is the convex hull of the set of $d\times d$ permutation matrices. Furthermore, the vertices (extreme points) of $B_d$ are the permutation matrices. 
\label{thm:birkhoff}
\end{lemma}
Combine the above two theorems, it is clear to see that there exists an optimal solution $P$ to the optimization problem (\ref{equ:doubly}) that is a $d\times d$ permutation matrix. This implies that we can reduce (\ref{equ:doubly}) to a minimum-weight perfect matching problem on a weighted complete bipartite graph. 
\begin{definition}[Minimum-weight perfect matching]
Let $G = (V, E)$ be an undirected graph with edge weight $w: E\to\RR_+$. A \emph{perfect matching} in $G$ is a set $M\subseteq E$ such that no two edges in $M$ have a vertex in common and every vertex from $V$ occurs as the endpoint of some edge in $M$. A matching $M$ is called a minimum-weight perfect matching if it is a perfect matching that has the minimum weight among all the perfect matchings of $G$. 
\end{definition}
For any $\lambda, \nu\in\RR_+^d$, we can construct a weighted $d$ by $d$ bipartite graph $G = (V_\lambda, V_\nu, E;w)$ as follows:
\begin{itemize}[noitemsep,topsep=0pt]
  \item   For each $\lambda_i$, construct a vertex $v_{\lambda_i}\in V_\lambda$, $\forall i$.
  \item   For each $\nu_j$, construct a vertex $v_{\nu_j}\in V_\nu$, $\forall j$.
  \item   For each pair $(v_{\lambda_i}, v_{\nu_j})$, construct an edge $e(v_{\lambda_i}, v_{\nu_j})\in E$ with weight $w(e(v_{\lambda_i}, v_{\nu_j})) = \lambda_i \nu_j$.
\end{itemize}
The following lemma relates the solution of the minimum weight matching to the solution of (\ref{equ:doubly}):
\begin{restatable}{lemma}{equivalence}
The minimum value of (\ref{equ:doubly}) is equal to the minimum weight of a perfect matching on $G = (V_\lambda, V_\nu, E, w)$.
\label{lemma:equal}
\end{restatable}
Surprisingly, we do not even need to run standard graph matching algorithms to solve our matching problem. Instead, Thm.~\ref{thm:matching} gives a closed form solution. 
\begin{restatable}{theorem}{matching}
Let $\lambda = (\lambda_1, \ldots, \lambda_d)$ and $\nu = (\nu_1, \ldots, \nu_d)$ with $\lambda_1\geq\cdots\geq\lambda_d$ and $\nu_1\leq\cdots\leq\nu_d$. The minimum-weight perfect matching on $G$ is $\pi^* = \{(v_{\lambda_i}, v_{\nu_i}): 1\leq i\leq d\}$ with the minimum weight $w(\pi^*) = \sum_{i=1}^d \lambda_i\nu_{i}$.
\label{thm:matching}
\end{restatable}
The permutation matrix that achieves the minimum weight is $P^* = I_d$ since $\pi^*(\lambda_i) = \nu_i$. Note that $P = K\circ K$, it follows that the optimal $K^*$ is also $I_d$. Hence we can solve for the optimal $U^*$ matrix by solving the equation $U^{*T}V = I_d$, which leads to $U^* = V$. Now plug in the optimal $K^* = I_d$ into (\ref{equ:partial}). The optimization w.r.t. $\Lambda$ decomposes into $d$ independent problems, each of which being a simple scalar optimization problem:
\begin{align}
& \underset{\lambda}{\text{minimize}} &&  \sum_{i=1}^d \lambda_i\nu_i - m\log \lambda_i \nonumber\\
& \text{subject to} && l \leq \lambda_i \leq u, \quad\forall i = 1, \ldots, d
\label{equ:newoptsigma}
\end{align}
Depending on whether the value $m/\nu_i$ is within the range $[l, u]$, the optimal solution $\lambda_i^*$ for each scalar minimization problem may take different forms. Define a hard-thresholding operator $\mathbb{T}_{[l,u]}(z)$ as follows:
\begin{equation}
\mathbb{T}_{[l, u]}(z) = 
\begin{cases}
l, & z < l \\
z, & l \leq z \leq u \\
u, & z > u
\end{cases}
\end{equation}
Using this hard-thresholding operator, we can express the optimal solution $\lambda_i^*$ as $\lambda_i^* = \mathbb{T}_{[l, u]}(m/\nu_i)$, which finishes the proof.

\subsubsection{Proof of Lemma~\ref{lemma:equal}}
\equivalence*
\begin{proof}
By Lemma~\ref{thm:lp} and Lemma~\ref{thm:birkhoff}, the optimal value is achieved when $P$ is a permutation matrix. Given a permutation matrix $P$, we can understand $P$ as a bijective mapping from the index of rows to the index of columns. Specifically, construct a permutation $\pi_P: [d]\to [d]$ from $P$ as follows. For each row index $i\in[d]$, $\pi_P(i) = j$ iff $P_{ij} = 1$. It follows that $\pi_P$ is a permutation of $[d]$ since $P$ is assumed to be a permutation matrix. The objective function in (\ref{equ:doubly}) can be written in terms of $\pi_P$ as
$$\lambda^T P \nu = \sum_{i=1}^d \lambda_i\nu_{\pi_P(i)}$$
which is exactly the weight of the perfect matching on $G(V_\lambda, V_\nu, E, w)$ given by $\pi_P$:
$$w(\pi_P) = w(\{(i, \pi_P(i)): 1\leq i\leq d\}) = \sum_{i=1}^d \lambda_i\nu_{\pi_P(i)}$$
Similarly, in the other direction, given any perfect matching $\pi: [d]\to [d]$ on the bipartite graph $G(V_\lambda, V_\nu, E, w)$, we can construct a corresponding permutation matrix $P_\pi$: $P_{\pi, ij} = 1$ iff $\pi(i) = j$, otherwise 0. Since $\pi$ is a perfect matching, the constructed $P_\pi$ is guaranteed to be a permutation matrix. 

Hence the problem of finding the optimal value of (\ref{equ:doubly}) is equivalent to finding the minimum weight perfect matching on the constructed bipartite graph $G(V_\lambda, V_\nu, E, w)$. Note that the above constructive process also shows how to recover the optimal permutation matrix $P_{\pi^*}$ from the minimum weight perfect matching $\pi^*$.
\end{proof}

\subsubsection{Proof of Theorem~\ref{thm:matching}}
Note that $\lambda = (\lambda_1, \ldots, \lambda_d)$ and $\nu = (\nu_1, \ldots, \nu_d)$ are assumed to satisfy $\lambda_1\geq\cdots\geq \lambda_d$ and $\nu_1\leq\cdots\leq\nu_d$. To make the discussion more clear, we first make the following definition of an \emph{inverse pair}.
\begin{definition}[Inverse pair]
Given a perfect match $\pi$ of $G(V_\lambda, V_\nu, E, w)$, $(\lambda_i, \lambda_j, \nu_k, \nu_l)$ is called an \emph{inverse pair} if $i\leq j$, $k\leq l$ and $(v_{\lambda_i}, v_{\nu_l})\in \pi$, $(v_{\lambda_j}, v_{\nu_k})\in\pi$.
\end{definition}
\begin{lemma}
Given a perfect match $\pi$ of $G(V_\lambda, V_\nu, E, w)$ and assuming $\pi$ contains an inverse pair $(\lambda_i, \lambda_j, \nu_k, \nu_l)$. Construct $\pi' = \pi\backslash\{(v_{\lambda_i}, v_{\nu_l}), (v_{\lambda_j}, v_{\nu_k})\}\cup\{((v_{\lambda_i}, v_{\nu_k}), (v_{\lambda_j}, v_{\nu_l})\}$. Then $w(\pi')\leq w(\pi)$.
\label{lemma:inverse}
\end{lemma}
\begin{proof}
Let us compare the weights of $\pi$ and $\pi'$. Note that since $i\leq j, k\leq l$, we have $\lambda_i\geq \lambda_j$ and $\nu_k\leq \nu_l$.
\begin{align*}
w(\pi') - w(\pi) &= (\lambda_i\nu_k + \lambda_j\nu_l) - (\lambda_i\nu_l + \lambda_j\nu_k) \\
& = (\lambda_i - \lambda_j)(\nu_k - \nu_l)\\
& \leq 0
\end{align*}
\begin{figure}[htb]
\centering
  \includegraphics[width=0.5\textwidth]{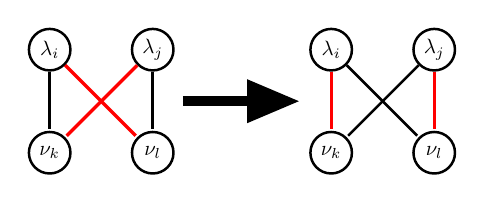}
\caption{Re-matching an inverse pair $(\lambda_i, \lambda_j, \nu_k, \nu_l) = \{(v_{\lambda_i}, v_{\nu_l}), (v_{\lambda_j}, v_{\nu_k})\}$ on the left side to a match with smaller weight $\{(v_{\lambda_i}, v_{\nu_k}), (v_{\lambda_j}, v_{\nu_l})\}$. Red color is used to highlight edges in the perfect matching. }
\label{fig:inverse}
\end{figure}
Intuitively, this lemma says that we can always decrease the weight of a perfect matching by re-matching an inverse pair. Fig.~\ref{fig:inverse} illustrates this process. It is worth emphasizing here that the above re-matching process only involves four nodes, i.e., $v_{\lambda_i}, v_{\nu_l}, v_{\lambda_j}$ and $v_{\nu_k}$. In other words, the other parts of the matching stay unaffected. 
\end{proof}
Using Lemma~\ref{lemma:inverse}, we are now ready to prove Thm.~\ref{thm:matching}:
\matching*
\begin{proof}
We will prove by induction. 
\begin{itemize}
  \item   \textbf{Base case}. The base case is $d = 2$. In this case there are only two valid perfect matchings, i.e., $\{(v_{\lambda_1}, v_{\nu_1}), (v_{\lambda_2}, v_{\nu_2})\}$ or $\{(v_{\lambda_1}, v_{\nu_2}), (v_{\lambda_2}, v_{\nu_1})\}$. Note that the second perfect matching $\{(v_{\lambda_1}, v_{\nu_2}), (v_{\lambda_2}, v_{\nu_1})\}$ is an inverse pair. Hence by Lemma~\ref{lemma:inverse}, $w(\{(v_{\lambda_1}, v_{\nu_1}), (v_{\lambda_2}, v_{\nu_2})\}) = \lambda_1\nu_1 + \lambda_2\nu_2 \leq \lambda_1\nu_2 + \lambda_2\nu_1 = w(\{(v_{\lambda_1}, v_{\nu_2}), (v_{\lambda_2}, v_{\nu_1})\})$. 
  \item   \textbf{Induction step}. Assume Thm.~\ref{thm:matching} holds for $d = n$. Consider the case when $d = n+1$. Start from any perfect matching $\pi$. Check the matches of node $v_{\lambda_{n+1}}$ and $v_{\nu_{n+1}}$. Here we have two subcases to discuss:
    \begin{itemize}
      \item     If $v_{\lambda_{n+1}}$ is matched to $v_{\nu_{n+1}}$ in $\pi$. Then we can remove nodes $v_{\lambda_{n+1}}$ and $v_{\nu_{n+1}}$ from current graph, and this reduces to the case when $n = d$. By induction assumption, the minimum weight perfect matching on the new graph is given by $\sum_{i=1}^n \lambda_i\nu_i$, so the minimum weight on the original graph is $\sum_{i=1}^n \lambda_i\nu_i + \lambda_{n+1}\nu_{n+1} = \sum_{i=1}^{n+1}\lambda_i\nu_i$. 
      \item     If $v_{\lambda_{n+1}}$ is not matched to $v_{\nu_{n+1}}$ in $\pi$. Let $v_{\nu_j}$ be the match of $v_{\lambda_{n+1}}$ and $v_{\lambda_i}$ be the match of $v_{\nu_{n+1}}$, where $i\neq n+1$ and $j\neq n+1$. In this case we have $i < n+1$ and $j < n+1$, so $(\lambda_i, \lambda_{n+1}, \nu_j, \nu_{n+1})$ forms an inverse pair by definition. By Lemma~\ref{lemma:inverse}, we can first re-match $v_{\lambda_{n+1}}$ to $v_{\nu_{n+1}}$ and $v_{\lambda_i}$ to $v_{\nu_j}$ to construct a new match $\pi'$ with $w(\pi')\leq w(\pi)$. In the new matching $\pi'$ we have the property that $v_{\lambda_{n+1}}$ is matched to $v_{\nu_{n+1}}$, and this becomes the above case that we have already analyzed, so we still have the minimum weight perfect matching to be $\sum_{i=1}^{n+1}\lambda_i\nu_i$.
    \end{itemize}
\end{itemize}
Intuitively, as shown in Fig.~\ref{fig:inverse}, an inverse pair corresponds to a cross in the matching graph. The above inductive proof basically works from right to left to recursively remove inverse pairs (crosses) from the matching graph. Each re-matching step in the proof will decrease the number of inverse pairs at least by one. The whole process stops until there is no inverse pair in the current perfect matching. Since the total number of possible inverse pairs, the above process can stop in finite steps. We illustrate the process of removing inverse pairs in Fig.~\ref{fig:matching}.
\begin{figure}[htb]
\centering
  \includegraphics[width=1.0\textwidth]{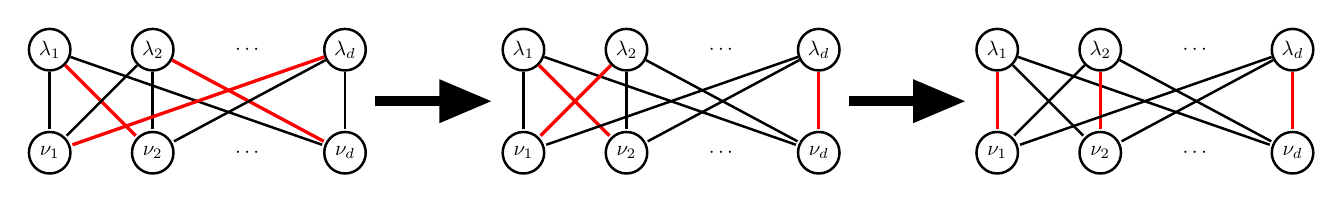}
\caption{The inductive proof works by recursively removing inverse pairs from the right side of the graph to the left side of the graph. The process stops until there is no inverse pair in the matching. Red color is used to highlight edges in the perfect matching.}
\label{fig:matching}
\end{figure}
\end{proof}

\end{document}